\pdfoutput=1
\documentclass[11pt]{article}

\usepackage{acl}
\usepackage{amsmath,amsfonts,bm}

\usepackage{amsmath,amsfonts,bm}
\usepackage{amsthm}
\newtheorem{theorem}{Theorem} %[chapter]

\newtheorem{lemma}[theorem]{Lemma}
\newtheorem{definition}[theorem]{Definition}
\def\figref#1{figure~\ref{#1}}
\def\Figref#1{Figure~\ref{#1}}

\def\tblref#1{table~\ref{#1}}
\def\secref#1{section~\ref{#1}}
\def\eqref#1{equation~\ref{#1}}
\def\1{\bm{1}}
\newcommand{\train}{\mathcal{D}}

\def\va{{\bm{a}}}
\def\vb{{\bm{b}}}

\def\vp{{\bm{p}}}

\def\vv{{\bm{v}}}
\def\vw{{\bm{w}}}
\def\vx{{\bm{x}}}
\def\vy{{\bm{y}}}
\def\vz{{\bm{z}}}

\def\evv{{v}}
\def\evw{{w}}

\def\mE{{\bm{E}}}

\def\mW{{\bm{W}}}

\DeclareMathAlphabet{\mathsfit}{\encodingdefault}{\sfdefault}{m}{sl}
\SetMathAlphabet{\mathsfit}{bold}{\encodingdefault}{\sfdefault}{bx}{n}

\def\gE{{\mathcal{E}}}

\def\sL{{\mathbb{L}}}

\def\sN{{\mathbb{N}}}

\def\sR{{\mathbb{R}}}
\def\sS{{\mathbb{S}}}

\usepackage{bbm}
\usepackage{subcaption}
\usepackage{xcolor}
\usepackage{makecell}
\usepackage{bbold}
\usepackage[most]{tcolorbox}
\usepackage{comment}
\usepackage{multirow}
\usepackage{multicol}
\usepackage{booktabs}
\usepackage{adjustbox}

\usepackage{bbm}
\usepackage{times}
\usepackage{latexsym}

\usepackage{xcolor}
\definecolor{yaushian}{RGB}{243, 101, 66}
\definecolor{ruohong}{RGB}{101, 66, 243}
\definecolor{yiming}{RGB}{255, 0, 0}

\newcommand{\high}[1]{\textcolor{red}{#1}}
\usepackage[T1]{fontenc}
\usepackage[utf8]{inputenc}

\usepackage{microtype}

\newcommand{\ourmodel}{\textsc{DEPL }}
\newcommand{\ourmodelshort}{\textsc{DEPL}}

\title{Long-tailed Extreme Multi-label Text  Classification \\ with Generated Pseudo Label Descriptions}

\author{Ruohong Zhang \and Yau-Shian Wang \\
  \texttt{ruohongz,yaushiaw@andrew.cmu.edu} \\\And
  Yiming Yang \\
  \texttt{yiming@cs.cmu.edu} \\ \AND
  Donghan Yu \\
  \texttt{dyu2@cs.cmu.edu} \\ \And
  Tom Vu \\
  \texttt{tom.m.vu@gmail.com} \\\And 
  Likun Lei \\
  \texttt{llei@flexport.com}
  }

\begin{document}
\maketitle

\begin{abstract}

%Extreme multi-label text classification (XMTC) is the task of tagging each document with the relevant labels in a large predefined label space. As the label frequency distribution is often highly skewed, a large portion of labels (namely the tail labels) have very few positive instances, posing a hard optimization problem for training the classification models. To address the data scarce issue in the long-tailed XMTC, this paper presents a neural retrieval framework which directly leverages the semantic matching between document text and label description. To further enhance quality of the label description, we propose to generate pseudo label descriptions from a trained bag-of-words (BoW) classifier, which demonstrates better classification performance under the tail label setting.
Extreme Multi-label Text Classification (XMTC) has been a tough challenge in machine learning research and applications due to the sheer sizes of the label spaces and the severe data scarce problem associated with the long tail of rare labels in highly skewed  distributions. This paper addresses the challenge of tail label prediction by proposing a novel approach, which combines the effectiveness of a trained bag-of-words (BoW) classifier in generating informative label descriptions under severe data scarce conditions, and the power of neural embedding based retrieval models in mapping input documents (as queries) to relevant label descriptions. The proposed approach achieves state-of-the-art performance on XMTC benchmark datasets and significantly outperforms the best methods so far in the tail label prediction. We also provide a theoretical analysis for relating the BoW and neural models w.r.t. performance lower bound.

\end{abstract}
\section{Introduction}
Extreme multi-label text classification (XMTC) is the task of tagging documents with relevant labels in a very large and often skewed candidate space. It has a wide range of applications, such as assigning subject topics to news or Wikipedia articles, tagging keywords for online shopping items, classifying industrial products for tax purposes, and so on. 

\begin{figure}[th!]
     \centering
     \includegraphics[width=\linewidth]{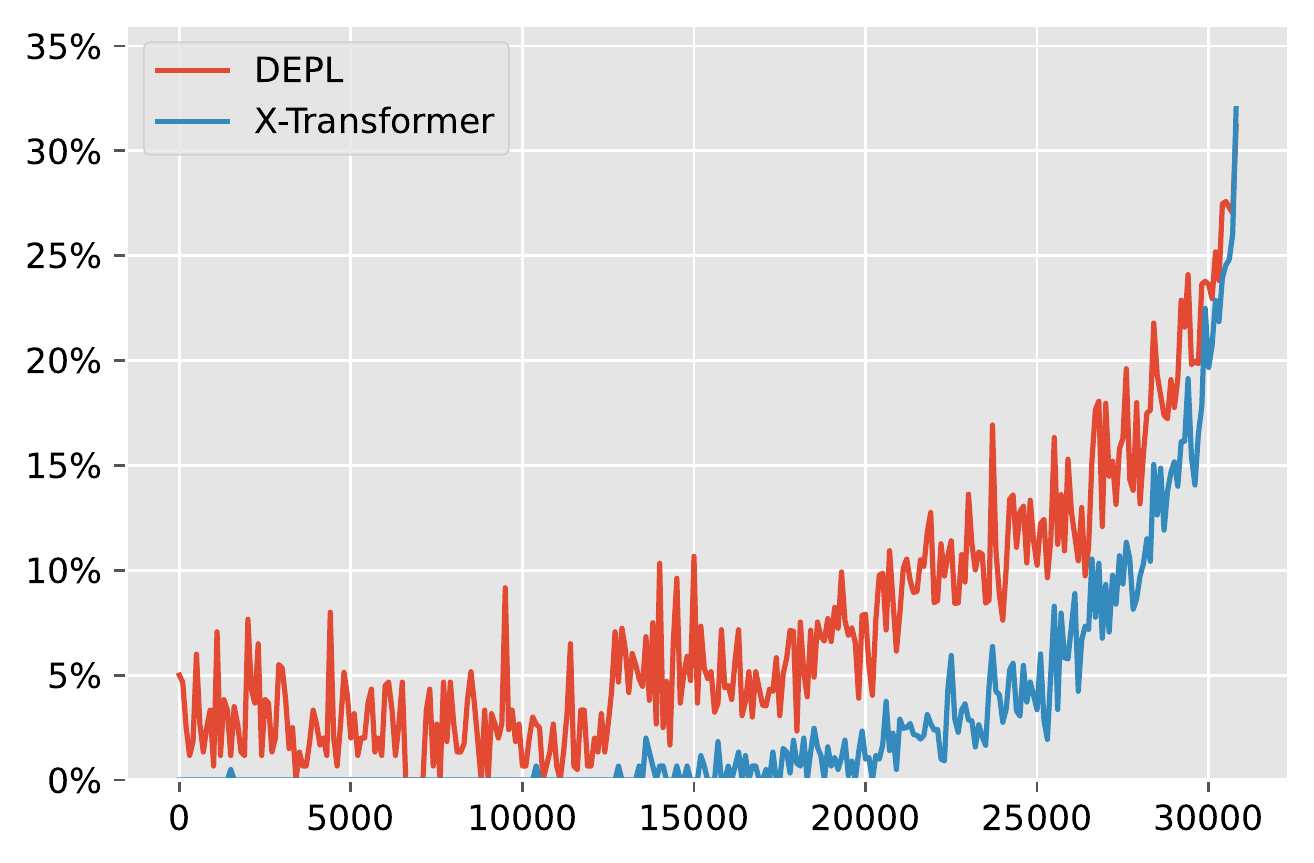}
     \caption{The classification performance of X-Transformer and DEPL (ours) on the Wiki10-31K dataset. The curves show the macro-averaged $F1@19$ scores of each system over the label bins (with 100 labels per bin).}
     \label{fig:data_dist}
\end{figure}

The most difficult part in solving the XMTC problem is to train classification models effectively for the rare labels in the long tail of highly skewed distributions, which suffers severely from the lack of sufficient training instances. 
Efforts addressing this challenge by the text classification community include hierarchical regularization methods for large margin classifiers \cite{gopal2013recursive}, Bayesian modeling of graphical or hierarchical dependencies among labels \cite{gopal2010multilabel, gopal2012bayesian}, clustering-based divide-and-conquer  strategies in resent neural classifiers \cite{chang2020taming, khandagale2019bonsai, prabhu2018parabel}, and so on. Despite the remarkable progresses made so far, and the problem is still very far from being well solved.  
\Figref{fig:data_dist} shows the performance of X-Transformer \cite{chang2020taming}, one of the 
state-of-the-art (SOTA) XMTC models with the best published result so far on the  Wiki10-31k benchmark dataset (with over 30,000 unique labels). The horizontal axis in this figure is the ranks of the labels sorted from rare to common, and the vertical axis is the text classification performance measured in macro-averaged $F1@19$ (higher the better) for binned labels (100 labels per bin).  The blue curve is the result of X-Transformer, which has the scores close to $0$ (worst possible score) for nearly half of the total labels.  
In other words, SOTA methods in XMTC still perform poorly in tail label prediction. We should point out that such poor performance of XMTC classifiers on tail labels has been largely overlooked in the recent literature, because benchmark evaluations of SOTA methods have typically used the metrics that are dominated by the system's performance on common labels, such as the micro-averaged precision at top-k (P@k). 
As a proper choice for evaluation of tail label prediction in \figref{fig:data_dist},  macro-averaged $F1@k$ is used, which gives the performance on each label an equal weight in the average. The red curve in the figure is the result of our new approach, being introduced the next.

In this paper, we seek solutions for tail label prediction from a new angle: we introduce a novel framework, namely the Dual Encoder with Pseudo Label (\ourmodelshort). It treats each input document as a query and uses a neural network model to retrieve relevant labels from the candidate space based on the textual descriptions of the labels.  The underlying assumption is, if the label descriptions are highly informative for text-based matching, then the retrieval system should be able to find relevant labels for input documents. Such a system would be particularly helpful for tail label prediction as the retrieval effectiveness does not necessarily rely on the availability of a large number of training instances, which what the tail labels are lacking. 

Now the key question is, how can we get a highly informative description for each label without human annotation?  In reality, class names are often available but they are typically one or two words, which cannot be sufficient for retrieval-based label prediction. Our answer is to use a relatively simple trained classifier, e.g., a linear support vector machine (SVM), to automatically generate an informative description for each label, which we call the pseudo description of the label.   The reason for us to choose a traditional classifier like linear SVM instead of a more modern neural model for label descriptions generation is that we want to better leverage the unsupervised statistics about word usage such as TF (term frequency within a document) and IDF (the inverse document frequencies within a document collection). Such unsupervised word features would be particularly helpful to alleviate the difficulty in classifier training under extreme data scarce conditions, and are easy and natural for traditional classifiers like SVM to leverage. The result of our approach (\ourmodelshort) is shown as the red curve in Figure 1, which significantly outperforms the blue curve of X-Transformer not only in the tail-label region but also in all other regions.  We also observed similar improvements by \ourmodel over strong baselines on other benchmark datasets (see \secref{sec:evaluation}).

Our main contributions can be summarized as the following:
\begin{enumerate}
\item We formulate the XMTC task as a neural retrieval problem, which enables us to alleviate the difficulty in tail label prediction by matching documents against label descriptions with advanced neural retrieval techniques. 
\item We enhance the retrieval system with pseudo label descriptions generated by a BoW classifier, which is proven to be highly effective for improving tail label prediction under severe data scarce conditions.
\item Our proposed method significantly and consistently outperforms strong baselines on multiple challenging benchmark datasets. Ablation tests provide in-depth analysis on different settings in our framework.
%, and outperformed all the strong baselines in our experiments.
%Our experiments show significant improvements by the proposed approach over the results of other SOTA methods on benchmark datasets, especially in tail label prediction.  
A theoretical analysis for relating the BoW and neural models w.r.t. performance lower bound is also provided.
\end{enumerate}

\section{Related Work}
% distributed signature
%\noindent \textbf{Sparse Classifier:} 
\paragraph{BoW (or Sparse) Classifier}
Traditional BoW classifiers rely on the bag-of-words features such as one-hot vector with tf-idf weights, which capture the word importance in a document. Since the feature is high dimenstional and sparse, we call the BoW classifiers the \textit{sparse classifiers}. Early examples include the one-vs-all SVM models such as DiSMEC \cite{babbar2017dismec}, ProXML \cite{babbar2019proxml} and PPDSparse \cite{yen2017ppdsparse}.
Later methods leverage the tree structure of the label space for more effective or scalable learning, such as Parabel \cite{prabhu2018parabel} and Bonsai \cite{khandagale2019bonsai}. Since tf-idf features rely on surface-level word matching, sparse classifiers tend to miss the semantic matching among lexical variants or related concepts in different wording.

%\cite{bao2019few} uses the distributional signatures utilizes their distributional signatures, characteristics of the underlying word distributions, 

\paragraph{Neural (or Dense) Classifier} 
% X-Transformer~\cite{chang2020taming}, APLC-XLNet~\cite{ye2020pretrained} and LightXML~\cite{jiang2021lightxml}
Neural models learn to capture the high level semantics of documents with dense feature embeddings. For this reason, we call them the \textit{dense classifiers}. The XML-CNN \cite{liu2017deep} and SLICE \cite{jain2019slice} employ the convolutional neural network on word embeddings for document representation. More recently, X-Transformer \cite{chang2020taming}, LightXML \cite{jiang2021lightxml} and APLC-XLNet \cite{ye2020pretrained} tames large pre-trained Transformer models to encode the input document into a fixed vector. AttentionXML \cite{you2018attentionxml} applies a label-word attention mechanism to calculate label-aware document embeddings, but it requires more computational cost proportional to the document length. In the above neural models, the feature extractor and the label embedding (randomly initialized) are jointly optimized via supervised signals. As we will show later, the document and the label embedding can be insufficiently optimized for the tail labels whose supervision signals are mostly negative.

%Rank-AE~\citep{wang2019ranking} uses the tf-idf feature as weights to aggregate word embeddings to form the document embedding, but the label embeddings are trained from scratch. 
%\noindent \textbf{Complement Neural Model with Sparse Features}:
\paragraph{Hybrid Approach:}
X-Transformer \citep{chang2020taming} complements neural model with sparse feature by concatenating tf-idf with the learned cluster-level neural embedding, which is an ensemble of the sparse and dense classifiers. Recent works in retrieval design unified systems to combine the sparse and dense features for better performance. SPARC~\cite{lee2019sparc} learns contextualized sparse feature indirectly via Transformer attention. 
COIL~\cite{gao2021coil} leverages lexical matching of contextualized BERT embeddings, and CLEAR~\cite{gao2004complementing} designs a residual-based loss function for the neural model to learn hard examples from a sparse retrieval model.
While we also combine the sparse feature with neural model, the classification setting does not assume predefined label description as in the retrieval setting.

\paragraph{Label Description:} 
When both the document text and label descriptions are available, the ranked-based multi-label classification is similar to the retrieval setting, where the dual encoder models~\cite{gao2021unsupervised, xiong2020approximate, luan2020sparse, karpukhin2020dense} have achieved SOTA performance in information retrieval on large benchmark datasets with millions of passages. The Siamese network \cite{dahiya2021siamesexml} for classification encodes both input documents and label descriptions under the assumption that high quality label descriptions are available. \citet{chai2020description} tries a generative model with reinforcement learning to produce extended label description with predefined label descriptions for initialization and uses cross attentions between input text document and output labels.  Although their ideas of utilizing label descriptions are attractive, the performance of those systems crucially depends on availability of predefined high-quality label descriptions, which is often difficult to obtain in real-world applications.  Instead, the realistic label descriptions are often short, noisy and insufficient for lexicon-matching based label prediction for input documents.  How to generate informative label descriptions without human efforts is thus an important problem, for which we offer an algorithmic solution in this paper.

\begin{figure}[t!]
     \centering
     \includegraphics[width=\linewidth]{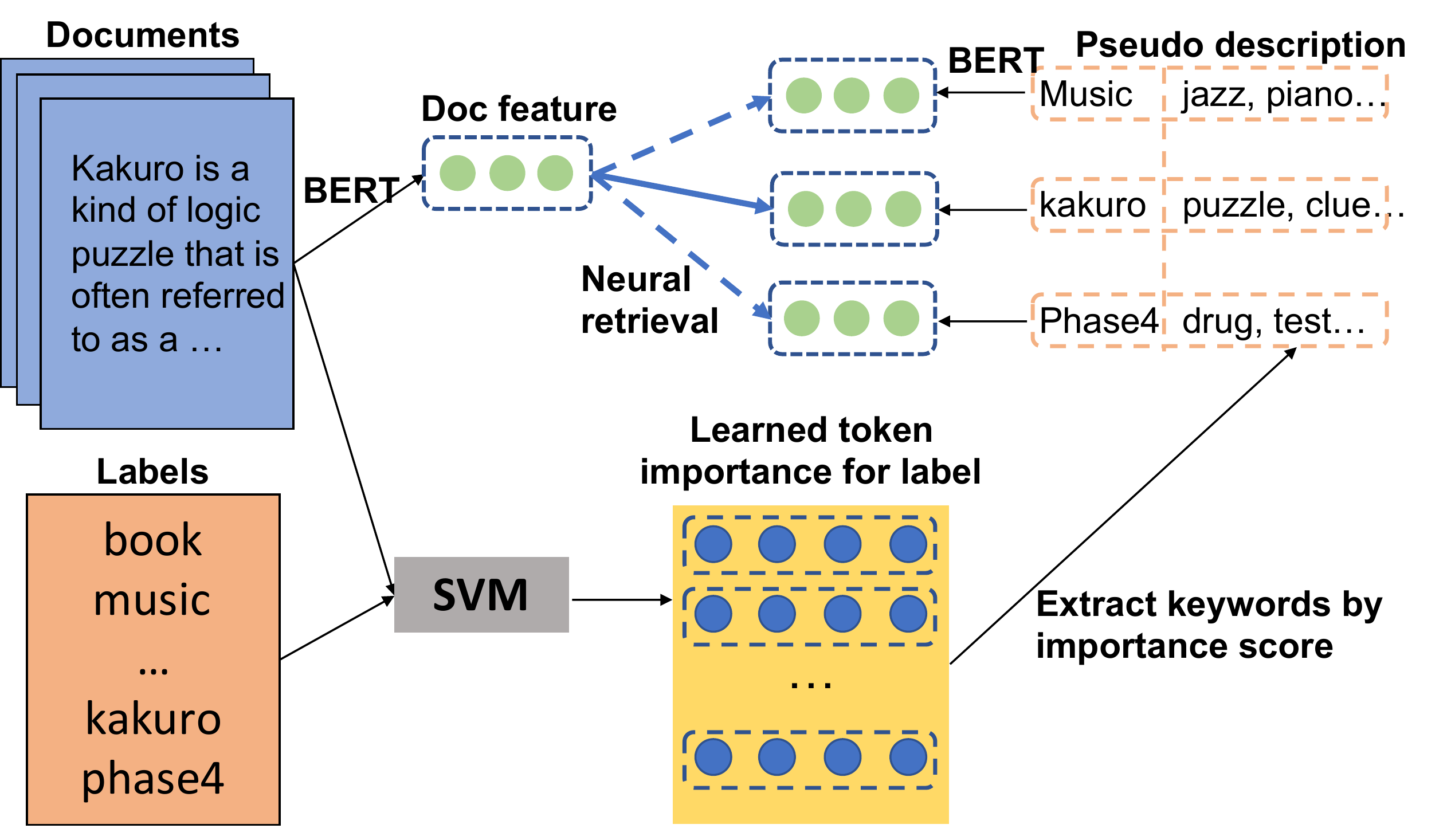}
     \caption{The proposed \ourmodel framework. First, we train a BoW classifier (SVM) and extract the top keywords from the label embeddings according to the learned token importance. Then, we concatenate the keywords with the original label names to form pseudo descriptions. Finally, we leverage the neural retrieval model to rank the labels according to semantic matching between document text and label descriptions.}
     \label{fig:model}
\end{figure}

\section{Proposed Method}
%From the analysis of document feature learning and observation 1, the sparse classifier has a strength in tail label prediction. Additionally, the learned label embeddings from a sparse classifier can be interpreted as importance of words in the corpus vocabulary for the label. In this way, we can extract the top k keywords from the label embedding as the pseudo label description. From the analysis of label feature learning, it is important to provide additional label side information for label embedding. In observation 2, the neural model may benefit from the extracted keywords that leverage the knowledge from the sparse classifier and generalize better with neural embeddings of the keywords.
\label{sec:model}
In this section,we first provide the preliminaries on multi-label text classification system. Then we discuss the generation of pseudo label description obtained from a sparse classifier, the design of \ourmodel from the retrieval perspective, and the enhanced classification system with the retrieval module. %we proposed a dual encoder framework that leverages the pairs of input document with the pseudo label descriptions extracted from a sparse classifier.

\subsection{Preliminaries}
Let $\train=\{\left(\mathbf{x}_i, \mathbf{y}_i\right)_{i=1}^{N_{\text{train}}} \}$ be the training data where $\mathbf{x}_i$ is the input text and $\mathbf{y}_i\in\{0, 1\}^L$ are the binary ground truth labels of size $L$.
Given an instance $\vx$ and a label $l$, a classification system produces a matching score of the text and label:
\begin{equation*}
    f(\vx, l) = \langle \phi(\vx), \vw_l \rangle
\end{equation*}
where $\phi(\vx)$ represent the document feature vector and $\vw_l$ represents the label embedding of $l$. The dot product $\langle \cdot, \cdot \rangle$ is used as the similarity function.

Typically, the label embedding $\vw_l$ is randomly initialized and trained from the supervised signal. While learning the embedding as free parameters is expressive when data is abundant, it could be difficult to be optimized under the data scarce situation.
%For the sparse classifier, the document feature is the BoW vector calcuated based on the statistics of local and global term frequency such as the tf-idf importance weight. During optimization, only the label embedding is learned from the supervised information of training document and label pairs. For the dense classifier, in comparison, both the document feature and the label embedding are jointly learned via back-propagation on the supervised signals. Typically, the label embeddings are randomly initialized and the document features are obtained from the output of a neural network, i.e. CNN \cite{liu2017deep}, LSTM \cite{you2018attentionxml} or the pre-trained Transformer models \cite{yang2019xlnet, devlin2018bert}.
% We first train a base neural classifier with feature extract $\phi_n(\vx)$ (initialized by BERT) and the label embeddings $\vu_l$ (randomly initialized): 
% \begin{equation}
%     f_{\text{dense}}(\vx, l) = \sigma(\langle \phi_n(\vx),  \vu_l \rangle)
% \end{equation}
% The matching score is then normalized by the sigmoid function to produce $p_l = \sigmoid(s_l)$, indicating the probability of the label being true.
% The probability is optimized by the binary cross entropy (BCE) loss:
% \begin{equation*}
%     \mathcal{L}_{\text{BCE}} = -\sum_{l=1}^L y_l\log p_l + (1-y_l)\log (1-p_l)
% \end{equation*}
\paragraph{Sketch of Method} To tackle the long-tailed XMTC problem, we propose \ourmodelshort, a neural retrieval framework with generated pseudo label descriptions, as shown in \figref{fig:model}. Instead of learning the label embedding from scratch, the retrieval module directly leverages the semantic matching between the document and label text, providing a strong inductive bias on tail label prediction. 

Furthermore, we will show in later sections that a BoW classifier with the statistical feature gives better performance in data scarce situation. The pseudo labels generated from a BoW classifier exploit such data heuristics, and the neural model complements the statistical information with semantic meaning for further improvement. Next, we introduce the components of our system in details.

\subsection{Pseudo Label Description}
Short label names are usually given in the benchmark dataset, such as the category name of an Amazon product or the could tag of a Wikipedia article.
However, the provided label name is usually noisy and ambiguous, of which the precise meaning need to be inferred from the document text (refer to \secref{sec:evaluation} for more detailed discussions). Therefore, we enhance the quality of label names by an augmentation with keywords generated from a sparse classifier.
While there are multiple choices of the sparse classifier, we pick the linear SVM model with tf-idf feature $\phi_t(\vx)$ for a balance between efficiency and performance:
    $$ f_{\text{sparse}}(\vx, l) = \langle \phi_t(\vx), \vw^{\text{svm}}_l \rangle  $$
The label embedding weight $\vw^{\text{svm}}_l$ is optimized with the hinge loss:
\begin{equation*}
    \mathcal{L}_{\text{hinge}} = \frac{1}{LB}\sum_{i=1}^B\sum_{l=1}^L \max (0, 1- \tilde{y}_l \cdot f_{\text{sparse}}(\vx_i, l))
\end{equation*}
where $\tilde{y}_{l}=2 y_{l}-1 \in\{-1,1\}$ and $B$ is the batch size.

For a train SVM model, $\vw^{\text{svm}}_l$ has the dimension equal to the vocabulary size and each value $\evw^{\text{svm}}_{li}$ of the label embedding denotes the importance of the token $i$ w.r.t label $l$. We select the top $k$ most important tokens (ranked according to the importance score) as keywords, which are appended to the original label name to form the pseudo label description:
\begin{equation*}
    \operatorname{pseudo\_label}(l) = \operatorname{label\_name}(l) \oplus \operatorname{keywords}(l)
\end{equation*}
where $\oplus$ is the append operation. The token importance learned in $\vw^{\text{svm}}_l$ is purely based on the statistics of word frequency. After we extract the tokens as keywords, we can additionally leverage the semantic meaning of them with the powerful neural network representations, which will be introduced in the next section.

%  Therefore, we propose to augment the label text with keywords generated from a trained sparse classifier, and it the pseudo label text. A second variant of our dual encoder will be fine-tuned on pseudo label descriptions. We optimize $\phi_n(.)$ for both the text ($\vx$) and keywords ($\vz_l$):
% \begin{equation}
%     f_{\text{KDE}}(\vx, l) = \sigma(\langle \phi_n(\vx), \phi_n(\vz_l)\rangle)
% \end{equation}
% The predicted probability is an average of two terms, where $\sigma(\langle \phi_n(\vx), \phi_n(\vz_l)\rangle)$ leverages the document and label semantic matching that benefits the tail label prediction, and $f_{\text{dense}}(\vx, l)$ is a dense classifier that is better optimized for head labels with sufficient training data.

%jointly learns the feature embedding $\phi_n(\vx; \theta)$ (parameterized by $\theta$) and the label embedding $\vw_l$. The system calculates a relevance score (logits) $s_l = \langle \phi_n(\vx), \vw_l \rangle$, 

% put in implementation details
%We use a different token type embedding for keywords.

\subsection{Retrieval Model with Label Text}
In the long-tailed XMTC problem, the training instances of document and label pair for tail labels are limited, and thus it is difficult to optimize for neural model with a large number of parameters. Instead, we use a dual encoder model \cite{gao2021unsupervised, xiong2020approximate, luan2020sparse, karpukhin2020dense} to leverage the semantic matching of document and label text. We use the BERT \cite{devlin2018bert} model as our design choice as the contextualized function, which is shared for both the document and label text encoding. The similarity between them is measured by the dot product:
\begin{equation*}
    f_{\text{dual}}(\vx, l) = \langle \phi_{\text{doc}}(\vx), \phi_{\text{label}}(\operatorname{text}(l)) \rangle
\end{equation*}
where $\operatorname{text}(l)$ is the textual information of the label $l$. When the textual information only includes the label name given in the dataset, we call the model \textbf{DE-ret}. Otherwise, when the textual information is the pseudo label, we call the model \textbf{DEPL}.

The document embedding $\phi_{\text{doc}}(\vx)$ is obtained from the CLS embedding of the BERT model followed by a linear pooling layer:
\begin{equation*}
    \phi_{\text{doc}}(\vx) = \mW_{doc} \cdot \operatorname{BERT}(\vx, \text{CLS}) + \vb_{doc}
\end{equation*}
where $\operatorname{BERT}(\operatorname{text}(l), \text{CLS})$ represents the contextualized embedding of the special CLS token. $\mW_{doc}$ and $\vb_{doc}$ are the weights and biases for the document pooler layer.

For the label embedding $\phi_{\text{label}}(\operatorname{text}(l))$, we take an average of the last hidden layer of BERT followed by a linear pooler layer:
\begingroup
\small
\begin{align}
    &\phi_{\text{label}} (\operatorname{text}(l)) = \mW_{label} \cdot \psi_{bert}(\operatorname{text}(l)) + \vb_{label} \\
    &\psi_{bert}(\operatorname{text}(l)) = \frac{1}{|\operatorname{text}(l)|}\sum_{j=1}^{|\operatorname{text}(l)|} \operatorname{BERT}(\operatorname{text}(l), j) \label{eq:avg_bert}
\end{align}
\endgroup
where $\operatorname{BERT}(\operatorname{text}(l), j)$ represents the contextualized embedding of the $j$-th token in $\operatorname{text}(l)$ obtained from the last hidden layer of the BERT model. $\mW_{label}$ and $\vb_{label}$ are the weights and biases for the label pooler layer. In the \eqref{eq:avg_bert}, the average embedding of label tokens yields better performance empirically than the CLS embedding for two potential reasons: 1) the keywords ranked by importance are not natural language and the CLS embedding may not effectively aggregate such type of information, and 2) the CLS embedding captures the global semantic of longer context while the average of token preserves more of the shorter label text meaning.

\paragraph{Learning with Negative Sampling}
In order to optimize $f_{\text{dual}}(\vx, l)$, we need to calculate the label embedding $\phi_{\text{label}}(\operatorname{text}(l))$. Since calculating all the label embeddings for each batch is both expensive and prohibitive by the memory limit, we resort to negative sampling strategies for in-batch optimization. Specifically, we sample a fix-sized subset of labels for each batch containing: 1) all the positive labels of the instances in the batch, 2) the top negative predictions by the sparse classifier as the hard negatives, and 3) the rest of the batch is filled with uniformly random sampled negatives labels.

Let $\sS_b$ be the subset of labels sampled for a batch. The objective for the dual encoder is:
\begin{equation*}
\begin{split}
    \mathcal{L}_{\text{dual}} = -\frac{1}{B |\sS_b|}\sum_{i=1}^B \Bigg( \Bigg.
    \sum_{p \in \vy_i^{+}} \log \sigma(f_{\text{dual}}(\vx_i, p)) \\
    + \sum_{n \in \sS_b \backslash \vy_i^+} \log \sigma((1 - f_{\text{dual}}(\vx_i, n))) \Bigg. \Bigg)
\end{split}
\end{equation*}
where $B$ is the batch size, $\vy_i^+$ is the postive labels for instance $i$, and $\sigma$ is the sigmoid function.

\subsection{Enhance Classification with Retrieval}
In the neural classification system, the label embedding is treated as free parameters to be learned from supervised data, which is more expressive for medium and head labels with abundant training instances. The dense classifier learns the function:
\begin{equation}
    f_{\text{dense}}(\vx, l) = \langle \phi_{\text{doc}}(\vx), \vw_l^{\text{neural}}\rangle
\end{equation}
We propose to enhance the classification model with the retrieval mechanism by jointly fine-tuning:
\begin{equation}
    f_{\text{cls-dual}}(\vx, l) = \frac{\sigma(f_\text{dual}(\vx, l)) +  \sigma(f_{\text{dense}}(\vx, l))}{2}
\end{equation}
The classification and retrieval modules share the same BERT encoder. We refer to the system as \textbf{DEPL+c}. The object function $\mathcal{L}_{\text{cls-dual}}$ is similar to $\mathcal{L}_{\text{dual}}$ except for replacing $f_{\text{dual}}$ with $f_{\text{\text{cls-dual}}}$.

The \textbf{DEPL+c} model looks like an ensemble of the two systems at the first sight, but there are two major differences: 1) As the BERT encoder is shared between the classification and retrieval modules, it doesn't significantly increase the number of parameters as in \cite{chang2020taming, jiang2021lightxml}; and 2) when the two modules are optimized together, the system can take advantages of both units according to the situation of head or tail label predictions.

%Intuitively, our model combines feature from two complementary sources: the corpus specific word importance statistics and neural embedding with word semantics. In the next section, we show a theoretical analysis on the performance on the neural model vs. sparse classifier.

%\section{Analysis of Current XMTC Systems}
\section{Theoretical Analyses of \ourmodel}
\label{sec:analysis}
In this section, we first %analyze the difficulty of the joint optimization paradigm of both document and label embedding in the recent neural classifiers under the tail label setting, and how the unsupervised statistics of sparse classifier could be a potential strength. 
discuss the relation between dense and sparse classifier and then we derive a lower bound performance of \ourmodel over the sparse classifier.

\subsection{Rethinking Dense and Sparse XMTC}
We analyze the document and label embedding optimization in the skewed label distribution from the gradient perspective. Specifically, recall the predicted probability optimized by the binary cross entropy (BCE) loss: 
% and show that: 1) the learned document feature lacks for the representation of tail label. 2) the tail label embedding is hard to encode meaningful information by supervised signals alone. The conclusion is also justified by empirical observations on tail label prediction.
\begin{equation*}
    \mathcal{L}_{\text{BCE}} = -\sum_{l=1}^L y_l\log p_l + (1-y_l)\log (1-p_l)
\end{equation*}
The derivative of $\mathcal{L}_{\text{BCE}}$ w.r.t the logits is:
\begin{equation*}
    \frac{\partial \mathcal{L}_{\text{BCE}}}{\partial s_l} =
    \begin{cases}
        p_{l} - 1 & \text{if } y_{l}=1 \\
        p_{l} & \text{otherwise} \\
    \end{cases}
\end{equation*}

\paragraph{Document Feature Learning} In multi-label classification, the document feature needs to reflect all the representations of relevant labels. In fact, the gradient describes the relation between feature $\phi_n(\vx)$ and label embedding $\vw_l$. By the chain rule, the gradient of $\mathcal{L}_{\text{BCE}}$ w.r.t the document feature is:
\begin{equation*}
    \frac{\partial \mathcal{L}_{\text{BCE}}(y_{l}, p_{l})}{\partial \phi_n(\vx)} =
    \begin{cases}
        (p_{l} - 1)\vw_l & \text{if } y_{l}=1 \\
        p_{l}\vw_l & \text{otherwise} \\
    \end{cases}
\end{equation*}
By optimizing parameters $\theta$ of feature extractor, 
the document representation is encourage to move away from the negative label representation, that is:
$$\phi_n(\vx; \theta^\prime) \leftarrow \phi_n(\vx; \theta) - \eta p_l \vw_l$$
where $\eta$ is the learning rate.
Since a tail label appears more often as negative labels and $\theta$ is shared for all the data, the feature extractor is unlikely to encode tail label information, making tail labels more difficult to be predicted. In comparison, the sparse feature like tf-idf is unsupervised from corpus statics, which does not suffer from this problem. The feature may still maintain the representation power to separate the tail labels. 

\paragraph{Label Feature Learning}
When the labels are treated as indices in a classification system, they are randomly initialized and learned from supervised signals. The gradients of $\mathcal{L}_{\text{BCE}}$ w.r.t the label feature is:
\begin{equation*}
    \frac{\partial \mathcal{L}_{\text{BCE}}(y_{l}, p_{l})}{\partial \vw_l} =
    \begin{cases}
        (p_{l} - 1) \phi_n(\vx) & \text{if } y_{l}=1 \\
        p_{l} \phi_n(\vx) & \text{otherwise} \\
    \end{cases}
\end{equation*}
The label embedding is updated by:
\begin{align*}
    \vw_l^\prime &= \vw_l + \frac{\eta}{N_{\text{train}}}\sum_{i: y_{il} = 1} (1 - p_{il}) \phi_n(\vx_i) \\
    & -\frac{\eta}{N_{\text{train}}} \sum_{i: y_{il}=0} p_{il}\phi_n(\vx_i)
\end{align*}
%After the optimization, the label embedding tends to include features from positive instances and exclude features from negative instances. 
As most of the instances are negative for a tail label, the update of tail label embedding is inundated with the aggregation of negative features, making it hard to encode distinctive feature reflecting its identity.  
Therefore, learning the tail label embedding from supervised signals alone can be very distracting. Although previous works leverage negative sampling to alleviate the problem~\cite{jiang2021lightxml, chang2020taming},
we argue that it is important to initialize the label embedding with the label side information.

\subsection{Analysis on Performance Lower Bound}
\label{sec:theory}
We will show that \ourmodel achieves a lower bound performance as the sparse classifier, given the selected keywords are important and the sparse classifier can separate the positive from the negative instances with non-trivial margin.

Let $\phi_t(\vx)$ be the normalized tf-idf feature vector of text with $\| \phi_t(\vx) \|_2 = 1$. The sparse label embeddings $\{ \vw_1, \ldots, \vw_L \}$ satisfies $\| \vw_l \|_2 \le 1, w_{li} > 0$. In fact, label embeddings can be transformed to satisfy the condition without affecting the prediction rank.  
Let $\vz_l$ be the top selected keywords from the sparse classifier, which is treated as the pseudo label. Define the sparse keyword embedding $\vv_l$ with $\evv_{li} = \evw_{li}$ if $i$ is an index of selected keywords and $0$ otherwise. 

In the following, we define the keyword importance and the classification error margin.
 \begin{definition} 
  For label $l$ and $\delta \ge 0$, the sparse keyword embedding $\vv_l$ is $\delta$-bounded if  $\langle \phi_t (\vx), \vv_l \rangle \ge \langle \phi_t(\vx), \vw_l  \rangle - \delta $. 
 \end{definition}
 
\begin{definition}
 For two labels $p$ and $n$, the error margin $\mu$ is the difference between the predicted scores $\mu(\phi (\vx), \vw_p, \vw_n) = \langle \phi (\vx), \vw_p - \vw_n \rangle$. 
\end{definition}
 
We state the main theorem below:
\begin{theorem}
\label{th:main}
Let $\phi_t(\vx)$ and $\phi_n(\vx)$ be the sparse and dense (dimension $d$) document feature, $\vw_l$ be the label embedding and  $\vz_l$ be the $\delta$-bounded keywords. For a positive label $p$, let $\sN_p = \{n_1, \ldots, n_{M_p} \}$ be a set of negative labels ranked lower than $p$. 
The error margin $\epsilon_i=\mu(\phi_t(\vx), \vw_p, \vw_{n_i})$ and $\epsilon = \min(\{ \epsilon_1, \ldots, \epsilon_{M_p} \})$. An error $\gE_i$ of the neural classifier occurs when 
\begin{equation}
    \mu(\phi_n(\vx), \phi_n(\vz_p), \phi_n(\vz_{n_i}) ) \le 0
\end{equation} 
The probability of any such error happening satisfies 
\begin{equation*}
    P(\gE_1 \cup \ldots \cup \gE_{M_p}) \le 4 {M_p} \exp (-\frac{(\epsilon - \delta)^2d}{50}) 
\end{equation*}
When $(\epsilon - \delta) \ge 10 \sqrt{\frac{\log M_p}{d}}$, the probability is bounded by $\frac{1}{M_p}$.
% \begin{equation*}
%     P(\gE_1 \cup \ldots \cup \gE_{M_p})
%     \le \frac{1}{M_p}
% \end{equation*}
\end{theorem}

\noindent \textbf{Discussion:} 
An error event occurs when the sparse model makes a correct prediction but the neural model doesn't. If the neural model avoids all such errors, the performance should be at least as good as the sparse model, and Theorem \ref{th:main} gives a bound of that probability.

The term $\delta$ measures the importance of selected keywords (smaller the more important), the error margin $\epsilon$ measures the difficulty the correctly predicted positive and negative pairs by the sparse model. The theorem states that the model achieves a lower bound performance as sparse classifier if the keywords are informative and error margin is non-trivial. Proofs and assumptions are in \secref{sec:proof} for interested readers.

\section{Evaluation Design}
\label{sec:experiment}

% \begin{table*}[th!]
%     \centering
%     \caption{$N_{\text{train}}$ and $N_{\text{test}}$ are the number of training and testing instances respectively. $F$ is the tf-idf feature size. $\bar{L}_d$ is the average number of labels per document. $L$ is the number of labels. For tail labels with $1\sim 9$ training instances, $p_{\text{tail}}^l$ is percentage of tail labels and $p_{\text{tail}}^d$ is the percentage of training instances covered by the tail labels. \label{tab:dataset}} 
%     \begin{tabular}{l|cccccccc}
%     \toprule
%     Dataset & $N_{train}$ & $N_{test}$ & $F$ & $\bar{L}_d$ & $L$ & $p_{\text{tail}}^l$ & $p_{\text{tail}}^d$ \\
%     \midrule
%     EURLex-4K & 15,539 &  3,809 & 34,932 & 5.30 & 3,956 & 63.48\% & 9.50\% \\
%     Wiki10-31K & 14,146 & 6,616 & 189,795 & 18.64 &  30,938 & 88.65\% & 27.06\%  \\
%     AmazonCat-13K & 1,186,239 & 306,782 & 200,000 & 5.04 & 13,330 & 29.53\% & 0.27\% \\
%     \bottomrule
%     \end{tabular}
% \end{table*}

\subsection{Datasets}

The benchmark datasets for our experiments are
EURLex-4K~\citep{10.1007/978-3-540-87481-2_4}, Wiki10-31K~\citep{zubiaga2012enhancing} and AmazonCat-13K~\citep{10.1145/2507157.2507163}. The Wiki10-31K and AmazonCat-13K were obtained from the Extreme Classification Repository\footnote{\url{http://manikvarma.org/downloads/XC/XMLRepository.html}}. As for EURLex-4K, we obtained an unstemmed version from the APLC-XLNet github\footnote{\url{https://github.com/huiyegit/APLC_XLNet.git}}. Table~\ref{tab:dataset} summarizes the corpus statistics. 

\begin{table}[th!]
    \centering
    \small
    \caption{Corpus Statistics: $N_{\text{train}}$ and $N_{\text{test}}$ are the number of training and testing instances respectively; $\bar{L}_d$ is the average number of labels per document, and $L$ is the number of unique labels in total;  $|\sL_{\text{tail}}|$ is the number of tail labels with $1\sim 9$ positive training instances. \label{tab:dataset}}  %$p_{\text{tail}}^l$ is percentage of tail labels and $p_{\text{tail}}^d$ is the percentage of training instances covered by the tail labels. \label{tab:dataset}} 
    \begin{adjustbox}{width=\columnwidth,center}
    \begin{tabular}{l|cccccc}
    \toprule
    Dataset & $N_{train}$ & $N_{test}$  & $\bar{L}_d$ & $L$ & $|\sL_{\text{tail}}|$ \\
    \midrule
    EURLex-4K & 15,539 &  3,809  & 5.30 & 3,956 & 2,413 \\
    Wiki10-31K & 14,146 & 6,616  & 18.64 &  30,938 & 26,545 \\
    AmazonCat-13K & 1,186,239 & 306,782  & 5.04 & 13,330 & 3,936 \\
    \bottomrule
    \end{tabular}
    \end{adjustbox}
\end{table}

%For the tail label evaluation, we consider 
For comparative evaluation of methods in tail label prediction, from each of the three corpora we extracted the subset of labels with $1 \sim 9$ positive training instances.  Those tail-label subsets include 
%because we assume the absolute number of training instance reflects the difficulty of optimization across datasets. The tail labels covers 
$63.48\%$, $88.65\%$ and $29.53\%$ of the total labels in the three corpora, respectively.  These numbers indicate that the label distributions are indeed highly skewed, with a heavy long tail in each corpus.

\subsection{Evaluation Metrics}
We use three metrics which are commonly used in XMTC evaluations,  namely, the micro-averaged $P@k)$, the PSP@k \cite{PSP, wei2021towards}, and the macro-averaged $F1@k$.
%We use the micro-averaging $P@k$ metric to evaluate the overall system performance. For the tail-label performance, we include the PSP metric that is previously used in tail-label evaluation \cite{PSP, wei2021towards}. We also use a macro-averaging $F1@k$ metric for a fine-grained analysis of tail-label with training instance $1 \sim 9$. 

%The \textbf{micro-averaging P@k} metric is used to evaluate a ranked list of predicted labels:
Given a ranked list of the predicted labels for each test document, the precision of the top-k labels is defined as: 
\begin{equation}
    P@k = \frac{1}{k} \sum_{i=1}^{k} \mathbbm{1}_{\vy_i^+}(p_i)
\end{equation}
where  $p_i$ is the $i$-th label in the list $\vp$ and $\mathbbm{1}_{\vy_i^+}$ is the indicator function. The average of the $P@k$ values across all the test-set documents is called \textbf{micro-averaged P@k}, which has been the most commonly used in XMTC evaluations. Since this metric gives an equal weight to the per-instance scores, the resulted average is dominated by the system's performance on the common (head) labels but not the tail labels.  %which can be dominated by the system performance on the head labels. 
In other words, the performance comparison in this metric cannot provide enough insights to the effectiveness of methods in tail label prediction.

As an alternative metric, \textbf{PSP@k} ~\cite{PSP}  re-weights the precision on each instance as: \begin{align*}
    PSP @ k=\frac{1}{k} \sum_{l=1}^{k} \frac{\mathbbm{1}_{\mathbf{y}}(\mathbf{p}_l)}{\operatorname{prop}(\mathbf{p}_l)}
\end{align*}
where  the propensity score $\operatorname{prop}(\mathbf{p}_l)$ in the denominator gives higher weights to tail labels.

%The macro-averaging metric gives each label an equal weight, which is used to evaluate the label level performance \cite{yang1999re}. We use the 
\textbf{Macro-averaged F1@k} \cite{yang1999re} is a metric that gives an equal weight to all the labels, including tail labels, head labels and any labels in the middle range. It is defined as the average of the label-specific $F1@k$ values, calculated based on a contingency table for each label, as shown in \tblref{tab:confusion_matrix}.  Specifically, the precision, recall and $F1$ for a predicted ranked list of length $k$ are computed as $\mathrm{P}= \frac{\mathrm{TP}}{\mathrm{TP} + \mathrm{FP}}, \mathrm{R}=\frac{ \mathrm{TP}}{ \mathrm{TP} + \mathrm{FN}}$, and $\operatorname{F1} = 2 \frac{P \cdot R }{P + R }$.
%which takes the harmonic average of prediction (P) and recall (R):
%\begin{equation}
%    \operatorname{F1} = 2 \frac{P \cdot R }{P + R }
%\end{equation} 
\begin{table}[ht]
    \centering
    \setlength{\tabcolsep}{4pt}
    \caption{Contingency table for label $l$.}
    \begin{adjustbox}{width=\columnwidth,center}
    \begin{tabular}{c|cc}
    & $l$ is true label & $l$ is not true label \\
    \hline
    $l$ is predicted & True Positive ($\mathrm{TP}_l$) & False Positive ($\mathrm{FP}_l$) \\
    $l$ is not predicted & False Negative ($\mathrm{FN}_l$) & True Negative ($\mathrm{TN}_l$) \\
    \end{tabular}
    \end{adjustbox}
    \label{tab:confusion_matrix}
\end{table}

%Given $N_{\text{test}}$ instances and $L$ labels, the macro-average computes the scores on individual category first ($\mathrm{F1}_l$), and then take an average over the corresponding the categories ($\mathrm{F1} = \frac{1}{|\sL|}\sum_{i \in \sL} \mathrm{F1}_l$), which reflects label level performance of the methods.

For micro-averaged P@k and PSP@k, we choose $k=1,3,5$ as in previous works. For macro-averaged F1@k, we choose $k=19$ for Wiki10-31K because it has an average of $18.64$ labels and $k=5$ for the rest datasets.

\subsection{Baselines}
For the tail label evaluation, our method is compared with the SOTA deep learning models including X-Transformer~\cite{chang2020taming}, XLNet-APLC~\cite{ye2020pretrained}, LightXML~\cite{jiang2021lightxml}, and AttentionXML~\cite{you2018attentionxml}. X-Transformer, LightXML, and XLNet-APLC employ pre-trained Transformers for document representation. We reproduced the results of single model (given in their implementation) predictions with BERT as the base model for LightXML, BERT-large for X-Transformer, XLNet for XLNet-APLC, and LSTM for AttentionXML. The AttentionXML utilizes label-word attention to generate label-aware document embeddings, while the other models generate fixed document embedding. 

For the overall prediction of all labels, we also include the baselines of sparse classifiers: DisMEC~\cite{babbar2017dismec}, PfastreXML~\cite{jain2016extreme}, Parabel~\cite{prabhu2018parabel}, Bonsai~\cite{khandagale2019bonsai}, and we use the published results for comparison. We provide an implementation of linear SVM model with our extracted tf-idf features as another sparse baseline, and a BERT-base classifier as another dense classifier (used to initialize \ourmodelshort). 

\subsection{Implementation Details}
For the sparse model, since the public available BoW feature doesn't have a vocabulary dictionary, we generate the tf-idf feature by ourselves. We tokenize and lemmatize the raw text with the Spacy~\cite{spacy2} library and extract the tf-idf feature with the Sklearn~\cite{scikit-learn} library, with unigram whose df is $>= 2$ and $<= 70\%$ of the total documents. We use the BERT model as the contextualize function for our retrieval model, which is initialized with a pretrained dense classifier. Specifically, we fine-tune a $12$ layer BERT-base model with different learning rates for the BERT encoder, BERT pooler and the classifier. The learning rates are $(1e-5$, $1e-4$, $1e-3)$ for Wiki10-31K and $(5e-5$, $1e-4$, $2e-3)$ for the rest datasets. For the negative sampling, we sample batch of 500 instances for Wiki10-31K, and $300$ for EURLex-4K and AmazonCat-13K. We include $10$ hard negatives predicted by the SVM model for each instances. We used learning rate $1e-5$ for fine-tuning the BERT of our retrieval model and $1e-4$ for the pooler and label embeddings. For the pseudo label descriptions, we concatenate the provided label description with the generated the top $20$ keywords. The final length is truncated up to $32$ after BERT tokenization. We use length $16$ of pseudo label description as the default setting for \ourmodelshort.

\begin{table*}[ht!]
    \centering
    \caption{Tail label prediction results of methods in $P@k$ on the labels with $1\sim9$ positive training instances.} %The performance of tail label prediction evaluated with the PSP metric. The first panel includes the SOTA deep learning models and the second panel includes our implementation of the sparse and dense classifiers as our building block model. Our proposed models \ourmodel and \ourmodelshort+c achieves the best performance.}
    \begin{adjustbox}{width=\linewidth,center}
    \begin{tabular}{c|ccc|ccc|ccc}
    \toprule
    \multicolumn{1}{c}{} & \multicolumn{3}{c}{EURLex-4K} &  \multicolumn{3}{c}{Wiki10-31K} & \multicolumn{3}{c}{AmazonCat-13K} \\
    \midrule
    Methods & PSP@1 & PSP@3 & PSP@5 & PSP@1 & PSP@3 & PSP@5 & PSP@1 & PSP@3 & PSP@5\\
    \midrule
    X-Transformer & 37.85 & 47.05 & 51.81  & 13.52 & 14.62 & 15.63 & 51.42 & 66.14 & 75.57 \\
    XLNet-APLC  & 42.21 & 49.83 & 52.88   & 14.43 & 15.38 & 16.47 & 52.55 & 65.11 & 71.36 \\
    LightXML & 40.54 & 47.56 & 50.50 & 14.09 & 14.87 & 15.52 & 50.70 & 63.14 & 70.13 \\
    AttentionXML & 44.20 & 50.85 & 53.87 & 14.49 & 15.65 & \bf 16.54 & 53.94 & 68.48 & 76.43 \\
    SVM  & 39.18 & 48.31 & 53.37 & 11.84 & 14.00 & 15.81 & 51.83 & 65.41 & 72.82 \\
    %BERT (ours) & 41.73 & 49.47 & 52.64 & 12.81 & 13.88 & 14.84  & 51.21 & 66.47 & 73.81 \\
    \midrule
    %DE-ret & 42.40 & 50.17 & 53.70             & 10.87 & 11.83 & 11.72 & 53.25 & 67.76 & 75.40 \\
    DEPL & \bf 45.60 & 52.28 & 53.52 & \bf 16.30 & \bf 16.26 & 16.27 & \bf 55.94 & \bf 70.01 & \bf 76.87 \\
    DEPL+c & 44.60 & \bf 52.74  & \bf 54.64 & 14.90 & 15.53 & 16.20 & 55.21 &  69.73 & 75.94 \\
    \bottomrule
    \end{tabular}
    \end{adjustbox}
    \label{tab:psp}
\end{table*}

\begin{figure*}[ht!]
    \centering
    \includegraphics[width=\linewidth]{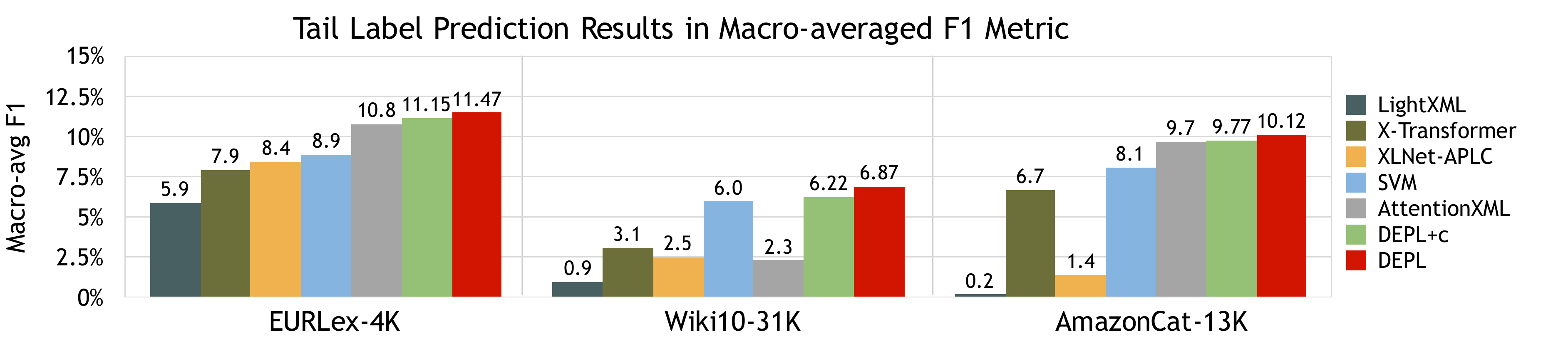}
    \caption{Tail-label prediction results in $F1@k$ on the labels with $1\sim9$ positive training instances, with $k=5$ for the EURLex-4K and AmazonCat-13K datasets, and $k=19$ for the Wiki10-31K dataset. } %The tail label evaluation by the macro F1@k metric averaged over labels with training instances $1\sim 9$, which cover $63.48\%, 88.65\%$ and $29.53\%$ of the label space of the three datasets. $k=5$ for the EURLex-4K and AmazonCat-13K datasets, and $k=19$ for the Wiki10-31K dataset. }
    \label{fig:tail_eval}
\end{figure*}

\section{Results}
\label{sec:evaluation}
Our experiments investigate the effectiveness of our model on the tail label prediction as well as its performance on the overall prediction. 

\subsection{Results in Tail Label Prediction} 
\paragraph{PSP Metric} The evaluation with the PSP metric is shown in \tblref{tab:psp}. We compare our model with the SOTA deep learning methods. Our proposed models \ourmodel and \ourmodelshort+c perform the best on the three benchmark datasets.

Since the Wiki10-31K dataset has more labels and the distribution is more skewed, the observation on the results is very interesting. On the one hand, the most frequent label covers more than $85\%$ of the training instances, so it is easier to predict that label than other lower frequencies labels. However, ranking the high frequency labels at the top doesn't give much gain on the PSP metric. 
On the other hand, each instance has an average of $19$ positive labels and ranking the subset of tail labels at the top could boost the PSP score. Since \ourmodel relies on the semantic matching between the document and label text, it is less affected by the dominating training pairs, and thus the PSP@1, PSP@3 beats the SOTA models by a larger margin. The \ourmodelshort+c achieves worse performance on this dataset, because the classification counterpart of the model would benefit more on the head label predictions and tend to rank the head labels at the top. 
%Another observation is that due to the large label space, the label name is more noisy. As a result, the pure retrieval model (DE-ret) performs poorly than the baseline models. This shows that extracting the pseudo label from an SVM model is help for enhancing the quality of label names.

For the EURLex-4K and AmazonCat-13K, we don't observe to much difference in \ourmodel and \ourmodel+c performance, and even \ourmodel+c can perform better on the EURLex-4K dataset. This maybe due to the improvement of head label predictions which can also boost the overall score in the micro-based metric.

\paragraph{Macro-F1 Metric} Although the PSP metric gives higher weight to the tail labels, it is a micro-averaged metric over the scores of each instance, which can still be affected by the high frequency labels that cover most of the instances. In comparison, the Macro-averaged F1 metric gives each label an equal weight and reflects more of the label level performance. We compare the Macro F1 metric averaged on the tail labels and the results are shown in \figref{fig:tail_eval}.

An interesting observation is that the SVM baseline achieves competitive results on the tail label predictions. We observe that SVM model can outperform all the pre-trained Transformer-based models on the tail label prediction, and outperform the AttentionXML on the Wiki10-31K dataset. We guess the reason is that the SVM utilizes the unsupervised statistical feature as document representation, which potentially suffers less from the data scarcity issue. The empirical result also serves as an evidence that the joint optimization of feature extractor and label embedding is difficult when data is limited. Among the deep learning baselines, the AttentionXML method performs the best on the tail label predictions, beating the SVM on $2$ out of the $3$ benchmark datasets. The reason could be that it utilizes the label to document word attention which is a local feature matching that benefits the tail label prediction.

Our proposed models perform the best on the Macro F1 metric with the \ourmodel model showing the best performance on all the $3$ benchmark datasets. We attribute the success of model to the retrieval module that focuses on the semantic matching between the document and label text. While the results evaluated in the Macro-averaged F1 metric show similar trend to that in the PSP metric, we do get some different conclusions with different metrics, i.e. the SVM model doesn't stand out under the PSP metric. Since the F1 metric is calculated specifically on the set of tail labels, it provides a more fine-grained result of tail label prediction. In comparison, the PSP metric also reflects the performance on the more common categories.
% For the EUR-Lex and AmazonCat-13K, the label name quality is better, so the pure retrieval model DE-ret also preforms well. The average number of positive label for each instance is around $5$ and the model can rank the positive labels at a decent accuracy, we don't observe to much difference in \ourmodel and \ourmodel+c performance.
% As we also observe, while the pure retrieval model achieves better performance than the SVM and other Transformer-based models on $2$ datasets, it still performs poorly on the Wiki10-31K dataset. That also proves that generating the pseudo label from the sparse classifier can enhance the text quality. Furthermore, the generated text allows \ourmodel to use the semantic information of the label keywords, which is ignored in the SVM model. This could be another reason why our model performs better than the SVM baseline on the Wiki10-31K dataset.

% overall
\begin{table*}[ht!]
    \centering
    \caption{The all-label prediction results of representative classification systems evaluated in the micro-avg P@k metric. The bold phase and underscore highlight the best and second best model performance. } %Our model shows competitive results over the sparse classifiers and the SOTA deep learning methods. The bold phase and underscore highlight the best and second best model performance.}
    \begin{tabular}{cc|ccc|ccc|ccc}
    \toprule
    \multicolumn{2}{c}{} & \multicolumn{3}{c}{EURLex-4K} &  \multicolumn{3}{c}{Wiki10-31K} & \multicolumn{3}{c}{AmazonCat-13K} \\
    \midrule
    & Methods & P@1 & P@3 & P@5 & P@1 & P@3 & P@5 & P@1 & P@3 & P@5\\
    \midrule
    \multirow{4}{*}{published results} & DisMEC  &  83.21 & 70.39 & 58.73 & 84.13 & 74.72 & 65.94 & 93.81 & 79.08 & 64.06 \\
    & PfastreXML & 73.14 & 60.16 & 50.54 &  83.57 & 68.61 & 59.10 & 91.75 & 77.97 & 63.68 \\
    & Parabel & 82.12 & 68.91 & 57.89 & 84.19 & 72.46 & 63.37 & 93.02 & 79.14 & 64.51 \\
    & Bonsai & 82.30 & 69.55 & 58.35 &  84.52 & 73.76 & 64.69 &  92.98 & 79.13 & 64.46 \\
    \midrule
     \multirow{6}{*}{replicated results} & AttentionXML & 85.12 & 72.80 & 61.01  & 86.46 & \underline{77.22} & 67.98  & 95.53 & 82.03 & 67.00 \\
    & X-Transformer & 85.46 & 72.87 & 60.79 & 87.12 & 76.51 & 66.69 & \underline{95.75} & \bf 82.46 & \underline{67.22} \\
    & XLNet-APLC & \bf 86.83 & \bf 74.34 & \underline{61.94} & \bf 88.99 & \bf 78.79 & \bf 69.79 & 94.56 & 79.78 & 64.59\\
    %BERT-APLC & 85.54 & 72.68 & 60.59 & \underline{88.54} & 77.21 & 67.43 & 94.49 & 79.74 & 64.46 \\
    & LightXML & 86.12 & \underline{73.87} & 61.67 & 87.39 & 77.02 & \underline{68.21} & 94.61 & 79.83 & 64.45\\
    %\midrule
    & SVM  & 83.44 & 70.62 & 59.08 & 84.61 & 74.64 & 65.89 & 93.20 & 78.89 & 64.14 \\
    & BERT   & 84.72 & 71.66 & 59.12 & 87.60 & 76.74 & 67.03 & 94.26 & 79.63 & 64.39 \\
    \midrule
    % DE-ret & 84.04 & 70.68 & 58.75 & 81.23 & 66.09 & 54.01 & 93.94 & 77.10 & 60.36 \\
     \multirow{2}{*}{our model results} & DEPL & 85.38 &  71.86 & 59.91 & 84.54 & 73.44 & 64.75 & 94.86 & 80.85 & 64.55 \\
    & DEPL+c & \underline{86.43} & 73.77 & \bf 62.19 & 87.32 & 77.05 & 67.39 & \bf 96.16 & \underline{82.23} & \bf 67.65 \\
    \bottomrule
    \end{tabular}
    \label{tab:micro_all}
\end{table*}

\begin{table*}[th!]
    \small
    \centering
    \caption{Examples of SVM generated keywords for tail labels from Wiki10-31K. The classifier is trained with only one positive training instance per label. For each label, we show the top $20$ extracted keywords. We manually highlights the meaningful keywords in red. \label{tab:description}}
    \begin{tabular}{cc|c}
    \toprule
    \bf Label Text & \bf \#training instance & \bf Top Keywords \\
    \midrule
    phase4  & 1 & \makecell{\high{trials} \high{clinical} protection personal directive processed data trial \high{drug} phase eu \\ processing \high{patients} sponsor \high{controller} legislation regulation art investigator \high{study} } \\
    \hline 
    ensemble & 1 &\makecell{ \high{boosting} kurtz ferrell weak \high{algorithms} \high{learners} \high{misclassified learner} kearns \high{ensemble} \\ charges bioterrorism indictment doj indict cae correlated 2004 \high{reweighted} boost } \\
    \hline 
    kakuro & 1 & \makecell{ \high{nikoli} kakuro \high{puzzles crossword clues entries entry values sums cells} \\ cross \high{digits} dell \high{solvers} racehorse guineas aa3aa digit clue kaji } \\
    \bottomrule
    \end{tabular}
\end{table*}

%\subsection{Overall Classification Performance}
\subsection{Results in All-label Prediction}
The performance of model evaluated on the micro-averaged P@k metric is reported in \tblref{tab:micro_all}. Our model is compared against the SOTA sparse and dense classifiers. The results on the EURLex-4K and AmazonCat-13K shows that \ourmodelshort+c achieves the best or second best performance compared to the all the SOTA models. On the Wiki10-31K dataset, our model performs worse than the other neural baseline, because the retrieval-based models have lower score evaluated on the P@k metric.

By comparing the overall performance with the tail label performance, we uncover a trade-off between the head label and tail label prediction. While our retrieval-based models achieve the best performance on the tail label evaluation on the Wiki10-31K dataset, it is worse under the micro-averaged P@5 metric, which could be dominated by the head labels. We want to emphasis that over $26,545$ ($88.65\%$) labels in the Wiki10-31K dataset belong to the tail labels with less than $10$ training instances, constituting a majority of the label space. The overall classification precision (P@k) only reflects a part of the success of a classification system, and the tail label evaluation is yet another part.

From the evaluation, we observe that the \ourmodelshort+c outperforms the retrieval-based counterpart \ourmodel and the BERT model as the baseline classifier. This shows that enhancing the classification system with the retrieval model improves the overall performance. The reason is that the dense classifier learns better representation when the training instances are abundant, while the retrieval system is better at matching the semantic of document and label text. Each of the modules captures a certain aspect of the data heuristic for text classification and a combination of them by sharing the BERT encoder yields better performance.

The sparse classifiers generally underperform the neural models and are comparable to our implement of SVM. We observe that \ourmodel can outperform the sparse models, even if the pseudo labels are extracted from the SVM classifier. The reason is that neural retrieval model can additionally leverage the keyword semantic information and correlation of them, which is ignored in the SVM classifier.

\begin{figure}[th!]
     \centering
     \includegraphics[width=\linewidth]{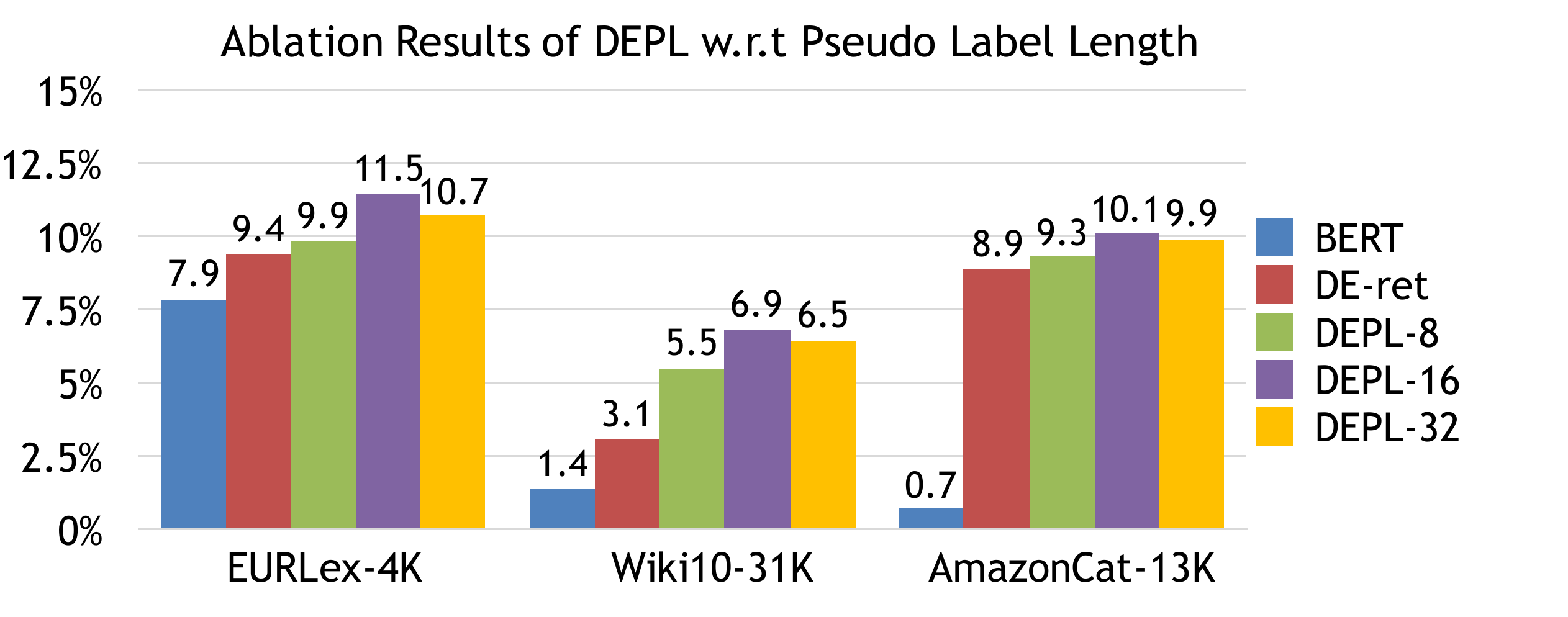}
     \caption{The ablation-test results of \ourmodel in Macro-averaged F1@k metric with varying length of pseudo label descriptions. }
     \label{fig:tail_ablation}
\end{figure}

\begin{table}[ht!]
    \centering
    \small
    \setlength{\tabcolsep}{3.5pt}
    \caption{Ablation-test results of \ourmodel under different training conditions.\label{tab:ablation}}
    \begin{adjustbox}{width=\columnwidth,center}
    \begin{tabular}{l|ccc|ccc}
    \toprule
    Methods & P@1 & P@3 & P@5 & PSP@1 & PSP@3 & PSP@5 \\
    \midrule
    \multicolumn{7}{c}{EUR-Lex}   \\
    \midrule
    DE-ret & -1.34 & -1.18 & -1.16 & -3.2 & -2.11 &  +0.18 \\
    w/o pre-train & -6.81 & -6.72 & -6.16 & -6.87 & -7.09 & -5.87  \\
    w/o neg & -2.52 & -2.63 & -2.2 & -3.59 & -3.66 & -1.9 \\
    5 hard negative &  -1.55 & -1.19 & -1.12 & -3.57 & -2.98 & -1.32  \\
    \midrule
    \multicolumn{7}{c}{Wiki10-31K}   \\
    \midrule
    DE-ret & -3.31 &  -7.35 & -10.74 & -5.43 & -4.43 & -4.55 \\
    w/o pre-train & -4.72 &  -8.3 & -10.9  & -1.56 & -1.36 & -1.27 \\
    w/o hard negative & -1.91 &  -5.54 & -3.73 & -0.25 & -0.53 & -0.72 \\
    5 hard negative  &  -0.71 & -2.77 & -2.96 & -0.32 & -0.35 & -0.27  \\
    \bottomrule
    \end{tabular}
    \end{adjustbox}
\end{table}

\subsection{Ablation Study}
\paragraph{Pseudo Label Description} 
In \figref{fig:tail_ablation}, we conduct an ablation test on the length of the pseudo label and the performance is measured by Macro-avg F1@k. The BERT classifier is included as a baseline with no label text information (length equals $0$). As we observe that the longer description of length $16$ performs the better, but when length is $32$, the performance doesn't increase as the text may become noisy with more unrelated keywords.

The DE-ret model is a pure retrieval baseline (avg length $3$) with only the label name. While it achieves good performance on the EURLex-4K and AmazonCat-13K datasets, it still performs poorly on the Wiki10-31K dataset. Another evidence is shown in \tblref{tab:ablation}, where the performance of DE-ret drops more significantly on the Wiki10-31K dataset. That shows that generating the keywords from the sparse classifier can enhance the text quality. Furthermore, the generated text allows \ourmodel to use the semantic information of the label keywords, which is ignored in the SVM model. This could be another reason why our model performs better than the SVM baseline on the Wiki10-31K dataset.
% For the EUR-Lex and AmazonCat-13K, the label name quality is better, so the pure retrieval model DE-ret also preforms well. The average number of positive label for each instance is around $5$ and the model can rank the positive labels at a decent accuracy, we don't observe to much difference in \ourmodel and \ourmodel+c performance.
% As we also observe, while the pure retrieval model achieves better performance than the SVM and other Transformer-based models on $2$ datasets, it still performs poorly on the Wiki10-31K dataset. That also proves that generating the pseudo label from the sparse classifier can enhance the text quality. Furthermore, the generated text allows \ourmodel to use the semantic information of the label keywords, which is ignored in the SVM model. This could be another reason why our model performs better than the SVM baseline on the Wiki10-31K dataset.

In \tblref{tab:description}, we pick the illustrative examples of the SVM generated keywords trained on the Wiki10-31K dataset for labels with only $1$ training example. We manually highlight the meaningful terms related to the label meaning. For example, the label name \textit{phase4} is ambiguous, whose meaning needs to be inferred from the corresponding document. From the keywords \textit{trial, clinical, drug ...}, we can understand it is about medical testing phase and the keywords can enhance the label semantic. In another example, \textit{kakuro} is a Japanese logic puzzle known as a mathematical crossword and the game play involves in adding number in the cells. Generating a description for \textit{kakuro} requires the background knowledge, but the keywords automatically learned from the sparse classifier provide the key concepts. Although not all the keywords can provide rich semantics to complement the original label name, they may serve as a context for the label to make it more distinguishable from others. 

\paragraph{Model Pre-training} We fine-tune our retrieval model on a pre-trained neural classifier (BERT) and \tblref{tab:ablation} shows that without using the pre-trained model, there is a significant drop in the precision and PSP metrics.
\paragraph{Negative Sampling}
We used the top negative predictions by the SVM model as the choice of hard negative labels. By default, we use $10$ hard negatives for each instance in the batch. In \tblref{tab:ablation}, we observe a performance drop when no hard negatives or only $5$ hard negatives are used for training.

% \subsection{SVM Result}
% \begin{figure*}[th!]
%      \centering
%      \includegraphics[width=\linewidth]{img/tail_eval.pdf}
%          \caption{tail label dist}
%          \label{fig:tail_dist}
%      \begin{subfigure}{\linewidth}
         
%     \end{subfigure}
    
%     \begin{subfigure}{\linewidth}
%         \includegraphics[width=\linewidth]{img/svm_tail_label.pdf}
%         \caption{SVM on tail label}
%         \label{fig:svm_tail}
%     \end{subfigure}
% \end{figure*}

    % \caption{In the skewed distribution of XMTC, tail labels with less than $10$ training instances cover a large portion, if not the majority, of the label space, but their training instances cover a only small percentage of the training set.}
\section{Conclusion}
In this paper, we propose a novel neural retrieval framework (DEPL) for the open challenge of tail-label prediction in XMTC.  By formulating the problem as to  capture the semantic mapping between input documents and system-enhanced label descriptions, DEPL combines the strengths of neural embedding based retrieval and the effectiveness of a large-marge BoW classifier in generating informative label descriptions under severe data sparse conditions.
Our extensive experiments on very large benchmark datasets show significant performance improvements by DEPL over strong baseline methods, especially in tail label description. 
%We further enhance the label text by generating pseudo descriptions using a BoW classifier. We show the improved performance of our model on the benchmark datasets over the strong neural baselines and derive a performance bound between our model and the BoW classifier.

% \bibliographystyle{acl_natbib}

\appendix
\onecolumn
\section{Appendix}
We include the assumptions and proofs of Theorem \ref{th:main}.
\label{sec:proof}
% \noindent \textbf{Notations:} Let $\phi_t(\vx)$ be the normalized tf-idf feature vector of text, s.t. $\| \phi_t(\vx) \|_2 = 1$. Let the learned sparse label embeddings be $\{ \vw_1, \ldots, \vw_L \}$ with $\| \vw_i \|_2 \le 1$ and $ \evw_{ij} \ge 0, i\in\{1, \ldots, L\}, j\in \{1, \ldots, V\}$. In fact, since we use a ranking metric, we can always normalize the label embeddings by 
% $\frac{\vw_l - \min( \{\evw_{ij}\} )}{\max  (\{ \| \vw_i - \min( \{\evw_{ij}\} ) \|_2 \}) }$,
% without changing the prediction rank. Let selected keywords be $\vz_l$, and $\vv_l$ be the keyword-selected label embedding with $\evv_{li} = \evw_{li}$ if $i$ is keyword and $0$ otherwise. Let $\phi_n(\vx) \in \mathbb{R}^d$ be the dense neural embedding. 

\paragraph{Assumptions}
Similar to \citet{luan2020sparse}, we treat neural embedding as fixed dense vector $\mE \in \mathbb{R}^{d\times v}$ with each entry sampled from a random Gaussian $N(0, d^{-1/2})$. $\phi_n(\vx) = \mE\phi_t(\vx)$ is weighted average of word embeddings by the sparse vector representation of text. According to the \textit{Johnson-Lindenstrauss (JL) Lemma}~\citep{johnson1984extensions, ben2002limitations}, even if the entries of $\mE$ are sampled from a random normal distribution, with large probability, $\langle \phi_t(\vx), \vv \rangle$ and $\langle \mE\phi_t(\vx), \mE\vv \rangle$ are close.

% The neural model encodes the document and label text into dense vectors. In our analysis, we assume $\phi(\vx) = \mE \phi_t(\vx)$ and $\phi(\vz_l) = \mE \vv_l$
% where $\mE \in \mathbb{R}^{d\times v}$ is a trivial word embedding with each entry sampled from a random normal distribution $N(0, d^{-1/2})$. This assumption gives a loose bound, if not a lower bound, for neural model performance.

%The \ourmodel model learns dense representations for both the text $\phi_n(\vx) \in \mathbb{R}^d$ and the keywords $\phi_n(\vz_l)$ where $z_{li} \in \sK_{v_l}^\delta$. 
%Although our model uses the BERT model to encode both text representation and label representation, 

%

\begin{lemma} 
\label{lemma:keyword}
Let $\vv$ be the $\delta$-bounded keyword-selected label embedding of $\vw$. For two labels $p, n$, the error margins satisfy:
$$
   | \mu( \phi_t(\vx), \vw_p, \vw_n ) 
   - \mu( \phi_t(\vx), \vv_p, \vv_n ) | \le \delta
$$
\end{lemma}
\begin{proof}
By the definition of $\delta$-bounded keywords,
% \begin{align}
%     & \langle \phi_t(\vx), \vw_p \rangle - \delta \le \langle \phi_t(\vx), \vv_p \rangle \le \langle \phi_t(\vx), \vw_p \rangle \\
%     & \langle \phi_t(\vx), \vw_n \rangle - \delta \le \langle \phi_t(\vx), \vv_n \rangle \le \langle \phi_t(\vx), \vw_n \rangle 
% \end{align}
% which is equivalent to 
\begin{align}
    & \langle \phi_t(\vx), \vw_p \rangle - \delta \le \langle \phi_t(\vx), \vv_p \rangle \le \langle \phi_t(\vx), \vw_p \rangle \label{eq:l1_1} \\ 
    & - \langle \phi_t(\vx), \vw_n \rangle \le -\langle \phi_t(\vx), \vv_n \rangle \le - \langle \phi_t(\vx), \vw_n \rangle + \delta \label{eq:l1_2}
\end{align}
Adding \eqref{eq:l1_1} and \eqref{eq:l1_2} finishes the proof:
\begin{equation}
   \langle \phi_t(\vx), \vw_p - \vw_n \rangle - \delta
   \le \langle \phi_t(\vx), \vv_p - \vv_n \rangle 
   \le  \langle \phi_t(\vx), \vw_p - \vw_n \rangle + \delta
\end{equation}
\end{proof}

\begin{lemma}
\label{lemma:main}
Let $\phi_t(\vx)$ and $\phi_n(\vx)$ be the sparse and dense (dimension $d$) document feature, $\vw_l$ be the label embedding and  $\vz_l$ be the $\delta$-bounded keywords. Let $p$ be a positive label and $n$ be a negative label ranked below $p$ be the sparse classifier. The error margin is $\epsilon=\mu(\phi_t(\vx), \vw_p, \vw_n)$. An error $\gE$ of neural classification occurs when $\mu(\phi_n(\vx), \phi_n(\vz_p), \phi_n(\vz_n) ) \le 0$. The probability $P(\gE) \le 4 \exp (-\frac{(\epsilon - \delta)^2d}{50})$.
\end{lemma}

\begin{proof}
By the JL Lemma \citep{ ben2002limitations}: 
For any two vectors $\va, \vb \in \sR^{v}$, let $\mE \in \sR^{d \times v}$ be a random matrix such that the entries are sampled from a random Gaussian. Then for every constant $\gamma > 0$:
\begin{equation}
 P\left(|\langle \mE \va, \mE \vb \rangle - \langle \va, \vb \rangle| \geq \frac{\gamma}{2}\left(\| \va \|^{2}+\|\vb\|^{2}\right)\right) 
\leq  4 \exp \left(-\frac{\gamma^2d}{8}\right)
\end{equation}
Let $\gamma = \frac{2}{5}(\epsilon - \delta)$, $\va = \phi_t(\vx)$ and $\vb = \vv_p - \vv_n$. Since $\| \va \|_2 = 1 $ and $\| \vb \|_2 \le (\|\vv_p \|_2 + \|\vv_n \|_2)^2 \le 4$, the JL Lemma gives 
\begin{align}
    &P\left(
    |\langle \mE \phi_t(\vx), \mE (\vv_p - \vv_n) \rangle - \langle \phi_t(\vx), \vv_p - \vv_n \rangle| \geq \epsilon - \delta  \right) \label{eq:jl} \\
    &\le 4 \exp (-\frac{(\epsilon - \delta)^2d}{50})
\end{align}
To complete the proof, we need to show $P(\gE) \le Eq. \ref{eq:jl}$: %From $\gE$, we can derive:
\begin{alignat}{3}
\gE
&\implies \quad    & | \langle \mE \phi_t(\vx), \mE (\vv_p - \vv_n) \rangle - \langle \phi_t(\vx), \vw_p - \vw_n \rangle | &\ge \epsilon \\
&\implies \quad    &| \langle \mE \phi_t(\vx), \mE (\vv_p - \vv_n) \rangle - \langle \phi_t(\vx), \vv_p - \vv_n \rangle | &\ge \epsilon - \delta \label{eq:l2_2}
\end{alignat}
where the \eqref{eq:l2_2} is derived by Lemma \ref{lemma:keyword}:
\begin{align}
&| \langle \mE \phi_t(\vx), \mE (\vv_p - \vv_n) \rangle - \langle \phi_t(\vx), \vv_p - \vv_n \rangle | \\
% = &| \langle \mE \phi_t(\vx), \mE (\vv_p - \vv_n) \rangle 
% - \langle \phi_t(\vx), \vw_p - \vw_n \rangle 
% + \langle \phi_t(\vx), \vw_p - \vw_n \rangle 
% - \langle \phi_t(\vx), \vv_p - \vv_n \rangle | \\
\ge &| \langle \mE \phi_t(\vx), \mE (\vv_p - \vv_n) \rangle 
- \langle \phi_t(\vx), \vw_p - \vw_n \rangle |
- \\
&|\langle \phi_t(\vx), \vw_p - \vw_n \rangle 
- \langle \phi_t(\vx), \vv_p - \vv_n \rangle | \\
\ge &\epsilon - \delta
\end{align}
Therefore $P(\gE) \le Eq. \ref{eq:jl}$, which completes the proof.
\end{proof}

% \begin{theorem}
% Let $\phi_t(\vx)$ and $\phi_n(\vx)$ be the sparse and dense (dimension $d$) feature, $\vw_l$ be the label embedding and  $\vz_l$ be the $\delta$-bounded keywords. For a positive label $p$, let $\sN_p = \{n_1, \ldots, n_{M_p} \}$ be a set of negative labels ranked lower than $p$. 
% The error margin $\epsilon_i=\mu(\phi_t(\vx), \vw_p, \vw_{n_i})$ and $\epsilon = \min(\{ \epsilon_1, \ldots, \epsilon_{M_p} \})$. An error event $\gE_i$ occurs when $\mu(\phi_n(\vx), \phi_n(\vz_p), \phi_n(\vz_{n_i}) ) \le 0$. 
% The probability of any such error happening satisfies 
% \begin{equation*}
%     P(\gE_1 \cup \ldots \cup \gE_{M_p}) \le 4 {M_p} \exp (-\frac{(\epsilon - \delta)^2d}{50}) 
% \end{equation*}
% When $(\epsilon - \delta) \ge 10 \sqrt{\frac{\log M_p}{d}}$, the probability is bounded by $\frac{1}{M_p}$.
% \end{theorem}
\noindent Proof of \textbf{Theorem} \ref{th:main}
    
\begin{proof}
The Lemma 2 shows that 
\begin{equation}
    P(\gE_i ) \le 4 \exp (-\frac{(\epsilon_i - \delta)^2d}{50}) \le 4 \exp (-\frac{(\epsilon - \delta)^2d}{50})
\end{equation}
By an union bound on the error events $\{ \gE_1, \gE_2, \ldots, \gE_{M_p} \}$, 
\begin{align}
     P(\gE_1 \cup \ldots \cup \gE_{M_p}) &\le \sum_{i=1}^{M_p} 4\exp(-\frac{(\epsilon_i - \delta)^2d}{50}) \\
     &= 4 {M_p} \exp (-\frac{(\epsilon - \delta)^2d}{50}) 
\end{align}
\end{proof}
When $(\epsilon - \delta)^2 \ge 10 \sqrt{\frac{\log {M_p}}{d}}$, we have $\exp (-\frac{(\epsilon - \delta)^2d}{50}) \le \frac{1}{4{M_p}^2}$ and therefore $P(\gE_1 \cup \ldots \cup \gE_{M_p}) \le \frac{1}{{M_p}}$.

\end{document}